%% file: AnalogicalBasedBO.tex
\documentclass[english]{article}
\usepackage[T1]{fontenc}
\usepackage[latin9]{inputenc}
\usepackage{amsmath}
\usepackage{amsthm}
\usepackage{amssymb}
\usepackage{xargs}[2008/03/08]

\makeatletter
\theoremstyle{plain}
\newtheorem{thm}{\protect\theoremname}
  \theoremstyle{plain}
  \newtheorem{cor}[thm]{\protect\corollaryname}

\usepackage{jmlr2e}
\usepackage{natbib}
\usepackage{flushend}
\usepackage{graphicx}
\usepackage{algolyx}

\jmlrheading{x}{2017}{y-z}{a/b}{c/d}{Trung Le, Khanh Nguyen, Tu Dinh Nguyen, and Dinh Phung}

\ShortHeadings{Analogical-based Bayesian Optimization}{Le et al}
\firstpageno{1}

\usepackage{colortbl} 
\definecolor{header_color}{rgb}{0.74,0.88,0.91}
\definecolor{even_color}{rgb}{0.9,0.9,0.9}
\definecolor{subheader_color}{rgb}{0.85,0.93,0.95}
\definecolor{childheader_color}{rgb}{1.0,0.93,0.87}

\definecolor{ccolor_best}{rgb}{1.0,0.9,0.9}
\definecolor{ccolor_wrong}{rgb}{1.0,0.85,0.85}

\newcolumntype{x}[1]{>{\centering\arraybackslash}p{#1}}

\makeatother

\usepackage{babel}
  \providecommand{\corollaryname}{Corollary}
\providecommand{\theoremname}{Theorem}

\begin{document}
\input{macros.tex}

\global\long\def\model{ABO}

\editor{TBA}

\title{Analogical-based Bayesian Optimization}

\author{\name{T}rung Le\email trung.l@deakin.edu.au \\
 \addr Centre for Pattern Recognition and Data Analytics, Australia\\
\AND\name{K}hanh Nguyen\email nkhanh@deakin.edu.au\\
\addr Centre for Pattern Recognition and Data Analytics, Australia\\
\AND \name{T}u Dinh Nguyen\email tu.nguyen@deakin.edu.au\\
\addr Centre for Pattern Recognition and Data Analytics, Australia\\
\AND \name{D}inh Phung\email dinh.phung@deakin.edu.au \\
 \addr Centre for Pattern Recognition and Data Analytics, Australia}

\maketitle
\begin{abstract}
Some real-world problems revolve to solve the optimization problem
$\max_{x\in\mathcal{X}}f\left(x\right)$ where $f\left(.\right)$
is a black-box function and $\mathcal{X}$ might be the set of non-vectorial
objects (e.g., distributions) where we can only define a symmetric
and non-negative similarity score on it. This setting requires a novel
view for the standard framework of Bayesian Optimization that generalizes
the core insightful spirit of this framework. With this spirit, in
this paper, we propose \emph{Analogical-based Bayesian Optimization}
that can maximize black-box function over a domain where only a similarity
score can be defined. Our pathway is as follows: we first base on
the geometric view of Gaussian Processes (GP) to define the concept
of influence level that allows us to analytically represent predictive
means and variances of GP posteriors and base on that view to enable
replacing kernel similarity by a more genetic similarity score. Furthermore,
we also propose two strategies to find a batch of query points that
can efficiently handle high dimensional data.
\end{abstract}

\section{Introduction}

\input{Intro.tex}

\section{GP-based Bayesian Optimization}

\input{GP_BayesOP.tex}

\section{Analogical-based Bayesian Optimization}

\input{Sim_BayesOP.tex}

\section{Application of Similarity-based Bayesian Optimization}

\input{Application.tex}

\section{Experiment}

\input{Experiment.tex}
\pagebreak{}

\bibliographystyle{plainnat}

\newpage{}

\appendix
\input{icml17_supp.tex}

\end{document}

%% file: macros.tex
\newcommand{\sidenote}[1]{\marginpar{\small \emph{\color{Medium}#1}}}

\global\long\def\se{\hat{\text{se}}}

\global\long\def\interior{\text{int}}

\global\long\def\boundary{\text{bd}}

\global\long\def\ML{\textsf{ML}}

\global\long\def\GML{\mathsf{GML}}

\global\long\def\HMM{\mathsf{HMM}}

\global\long\def\support{\text{supp}}

\global\long\def\new{\text{*}}

\global\long\def\stir{\text{Stirl}}

\global\long\def\mA{\mathcal{A}}

\global\long\def\mB{\mathcal{B}}

\global\long\def\mF{\mathcal{F}}

\global\long\def\mK{\mathcal{K}}

\global\long\def\mH{\mathcal{H}}

\global\long\def\mX{\mathcal{X}}

\global\long\def\mZ{\mathcal{Z}}

\global\long\def\mS{\mathcal{S}}

\global\long\def\Ical{\mathcal{I}}

\global\long\def\mT{\mathcal{T}}

\global\long\def\Pcal{\mathcal{P}}

\global\long\def\dist{d}

\global\long\def\HX{\entro\left(X\right)}
 \global\long\def\entropyX{\HX}

\global\long\def\HY{\entro\left(Y\right)}
 \global\long\def\entropyY{\HY}

\global\long\def\HXY{\entro\left(X,Y\right)}
 \global\long\def\entropyXY{\HXY}

\global\long\def\mutualXY{\mutual\left(X;Y\right)}
 \global\long\def\mutinfoXY{\mutualXY}

\global\long\def\given{\mid}

\global\long\def\gv{\given}

\global\long\def\goto{\rightarrow}

\global\long\def\asgoto{\stackrel{a.s.}{\longrightarrow}}

\global\long\def\pgoto{\stackrel{p}{\longrightarrow}}

\global\long\def\dgoto{\stackrel{d}{\longrightarrow}}

\global\long\def\lik{\mathcal{L}}

\global\long\def\logll{\mathit{l}}

\global\long\def\vectorize#1{\mathbf{#1}}

\global\long\def\vt#1{\mathbf{#1}}

\global\long\def\gvt#1{\boldsymbol{#1}}

\global\long\def\idp{\ \bot\negthickspace\negthickspace\bot\ }
 \global\long\def\cdp{\idp}

\global\long\def\das{\triangleq}

\global\long\def\id{\mathbb{I}}

\global\long\def\idarg#1#2{\id\left\{  #1,#2\right\}  }

\global\long\def\iid{\stackrel{\text{iid}}{\sim}}

\global\long\def\bzero{\vt 0}

\global\long\def\bone{\mathbf{1}}

\global\long\def\boldm{\boldsymbol{m}}

\global\long\def\bff{\vt f}

\global\long\def\ba{\boldsymbol{a}}

\global\long\def\bb{\boldsymbol{b}}

\global\long\def\bB{\boldsymbol{B}}

\global\long\def\bx{\boldsymbol{x}}

\global\long\def\bl{\boldsymbol{l}}

\global\long\def\bu{\boldsymbol{u}}

\global\long\def\bo{\boldsymbol{o}}

\global\long\def\bh{\boldsymbol{h}}

\global\long\def\bs{\boldsymbol{s}}

\global\long\def\bz{\boldsymbol{z}}

\global\long\def\xnew{y}

\global\long\def\bxnew{\boldsymbol{y}}

\global\long\def\bX{\boldsymbol{X}}

\global\long\def\tbx{\tilde{\bx}}

\global\long\def\by{\boldsymbol{y}}

\global\long\def\bY{\boldsymbol{Y}}

\global\long\def\bZ{\boldsymbol{Z}}

\global\long\def\bU{\boldsymbol{U}}

\global\long\def\bv{\boldsymbol{v}}

\global\long\def\bn{\boldsymbol{n}}

\global\long\def\bV{\boldsymbol{V}}

\global\long\def\bI{\boldsymbol{I}}

\global\long\def\bw{\vt w}

\global\long\def\balpha{\gvt{\alpha}}

\global\long\def\bbeta{\gvt{\beta}}

\global\long\def\bmu{\gvt{\mu}}

\global\long\def\btheta{\boldsymbol{\theta}}

\global\long\def\bsigma{\boldsymbol{\sigma}}

\global\long\def\blambda{\boldsymbol{\lambda}}

\global\long\def\bgamma{\boldsymbol{\gamma}}

\global\long\def\bpsi{\boldsymbol{\psi}}

\global\long\def\bphi{\boldsymbol{\phi}}

\global\long\def\bPhi{\boldsymbol{\Phi}}

\global\long\def\bpi{\boldsymbol{\pi}}

\global\long\def\bomega{\boldsymbol{\omega}}

\global\long\def\bepsilon{\boldsymbol{\epsilon}}

\global\long\def\btau{\boldsymbol{\tau}}

\global\long\def\realset{\mathbb{R}}

\global\long\def\realn{\realset^{n}}

\global\long\def\integerset{\mathbb{Z}}

\global\long\def\natset{\integerset}

\global\long\def\integer{\integerset}

\global\long\def\natn{\natset^{n}}

\global\long\def\rational{\mathbb{Q}}

\global\long\def\rationaln{\rational^{n}}

\global\long\def\complexset{\mathbb{C}}

\global\long\def\comp{\complexset}

\global\long\def\compl#1{#1^{\text{c}}}

\global\long\def\and{\cap}

\global\long\def\compn{\comp^{n}}

\global\long\def\comb#1#2{\left({#1\atop #2}\right) }

\global\long\def\nchoosek#1#2{\left({#1\atop #2}\right)}

\global\long\def\param{\vt w}

\global\long\def\Param{\Theta}

\global\long\def\meanparam{\gvt{\mu}}

\global\long\def\Meanparam{\mathcal{M}}

\global\long\def\meanmap{\mathbf{m}}

\global\long\def\logpart{A}

\global\long\def\simplex{\Delta}

\global\long\def\simplexn{\simplex^{n}}

\global\long\def\dirproc{\text{DP}}

\global\long\def\ggproc{\text{GG}}

\global\long\def\DP{\text{DP}}

\global\long\def\ndp{\text{nDP}}

\global\long\def\hdp{\text{HDP}}

\global\long\def\gempdf{\text{GEM}}

\global\long\def\Gumbel{\text{Gumbel}}

\global\long\def\Uniform{\text{Uniform}}

\global\long\def\Mult{\text{Mult}}

\global\long\def\rfs{\text{RFS}}

\global\long\def\bernrfs{\text{BernoulliRFS}}

\global\long\def\poissrfs{\text{PoissonRFS}}

\global\long\def\grad{\gradient}
 \global\long\def\gradient{\nabla}

\global\long\def\partdev#1#2{\partialdev{#1}{#2}}
 \global\long\def\partialdev#1#2{\frac{\partial#1}{\partial#2}}

\global\long\def\partddev#1#2{\partialdevdev{#1}{#2}}
 \global\long\def\partialdevdev#1#2{\frac{\partial^{2}#1}{\partial#2\partial#2^{\top}}}

\global\long\def\closure{\text{cl}}

\global\long\def\cpr#1#2{\Pr\left(#1\ |\ #2\right)}

\global\long\def\var{\text{Var}}

\global\long\def\Var#1{\text{Var}\left[#1\right]}

\global\long\def\cov{\text{Cov}}

\global\long\def\Cov#1{\cov\left[ #1 \right]}

\global\long\def\COV#1#2{\underset{#2}{\cov}\left[ #1 \right]}

\global\long\def\corr{\text{Corr}}

\global\long\def\sst{\text{T}}

\global\long\def\SST{\sst}

\global\long\def\ess{\mathbb{E}}

\global\long\def\Ess#1{\ess\left[#1\right]}

\newcommandx\ESS[2][usedefault, addprefix=\global, 1=]{\underset{#2}{\ess}\left[#1\right]}

\global\long\def\fisher{\mathcal{F}}

\global\long\def\bfield{\mathcal{B}}
 \global\long\def\borel{\mathcal{B}}

\global\long\def\bernpdf{\text{Bernoulli}}

\global\long\def\betapdf{\text{Beta}}

\global\long\def\dirpdf{\text{Dir}}

\global\long\def\gammapdf{\text{Gamma}}

\global\long\def\gaussden#1#2{\text{Normal}\left(#1, #2 \right) }

\global\long\def\gauss{\mathbf{N}}

\global\long\def\gausspdf#1#2#3{\text{Normal}\left( #1 \lcabra{#2, #3}\right) }

\global\long\def\multpdf{\text{Mult}}

\global\long\def\poiss{\text{Pois}}

\global\long\def\poissonpdf{\text{Poisson}}

\global\long\def\pgpdf{\text{PG}}

\global\long\def\wshpdf{\text{Wish}}

\global\long\def\iwshpdf{\text{InvWish}}

\global\long\def\nwpdf{\text{NW}}

\global\long\def\niwpdf{\text{NIW}}

\global\long\def\studentpdf{\text{Student}}

\global\long\def\unipdf{\text{Uni}}

\global\long\def\transp#1{\transpose{#1}}
 \global\long\def\transpose#1{#1^{\mathsf{T}}}

\global\long\def\mgt{\succ}

\global\long\def\mge{\succeq}

\global\long\def\idenmat{\mathbf{I}}

\global\long\def\trace{\mathrm{tr}}

\global\long\def\argmax#1{\underset{_{#1}}{\text{argmax}} }

\global\long\def\argmin#1{\underset{_{#1}}{\text{argmin}\ } }

\global\long\def\diag{\text{diag}}

\global\long\def\norm{}

\global\long\def\spn{\text{span}}

\global\long\def\vtspace{\mathcal{V}}

\global\long\def\field{\mathcal{F}}
 \global\long\def\ffield{\mathcal{F}}

\global\long\def\inner#1#2{\left\langle #1,#2\right\rangle }
 \global\long\def\iprod#1#2{\inner{#1}{#2}}

\global\long\def\dprod#1#2{#1 \cdot#2}

\global\long\def\norm#1{\left\Vert #1\right\Vert }

\global\long\def\entro{\mathbb{H}}

\global\long\def\entropy{\mathbb{H}}

\global\long\def\Entro#1{\entro\left[#1\right]}

\global\long\def\Entropy#1{\Entro{#1}}

\global\long\def\mutinfo{\mathbb{I}}

\global\long\def\relH{\mathit{D}}

\global\long\def\reldiv#1#2{\relH\left(#1||#2\right)}

\global\long\def\KL{KL}

\global\long\def\KLdiv#1#2{\KL\left(#1\parallel#2\right)}
 \global\long\def\KLdivergence#1#2{\KL\left(#1\ \parallel\ #2\right)}

\global\long\def\crossH{\mathcal{C}}
 \global\long\def\crossentropy{\mathcal{C}}

\global\long\def\crossHxy#1#2{\crossentropy\left(#1\parallel#2\right)}

\global\long\def\breg{\text{BD}}

\global\long\def\lcabra#1{\left|#1\right.}

\global\long\def\lbra#1{\lcabra{#1}}

\global\long\def\rcabra#1{\left.#1\right|}

\global\long\def\rbra#1{\rcabra{#1}}

%% file: Intro.tex
Bayesian optimization (BO) has emerged as a powerful solution for
these varied design problems \citep{shahriari2016taking}. BO has
been widely applied to a mixed variety of real-world problems from
interactive user interfaces \citep{Brochu:2010}, robotics \citep{LizotteWBS07,Martinez-Cantin-RSS-07},
environmental monitoring \citep{MarchantR12}, information extraction
\citep{pmlr-v33-wang14d}, combinatorial optimization \citep{Hutter:2011,wang2013bayesian},
automatic machine learning \citep{NIPS2011_4443,pmlr-v33-hoffman14,NIPS2012_4522,NIPS2013_5086},
sensor networks \citep{icml2010_SrinivasKKS10}, adaptive Monte Carlo
(MC) \citep{pmlr-v22-mahendran12}, experimental design \citep{azimi2012hybrid},
and reinforcement learning \citep{journals/corr/abs-1012-2599} to
name a few. 

Fundamentally, BO is a sequential model approach to solve the optimization
$\max_{x}f\left(x\right)$ with regard to a black-box function $f\left(.\right)$,
wherein one is capable of querying the value of $f\left(x\right)$
for any given $x$. We initially place prior belief on the function
$f\left(.\right)$ which could be a GP. Susequently, this belief is
updated using queried data points and their labels. To decide which
point should be queried next, we recruit the acquisition which is
closely related to the updated belief. A good acquisition function
must ballance the exploitation and exploration to guarantee suggesting
points with high values in low density area.

In this paper, we propose Analogical-based Bayesian Optimization that
can optimize the black-box function $f\left(x\right)$ on a domain
$\mathcal{X}$ where we can endow a non-negative and symmetric similarity
score function $S\left(.,.\right)$. Our pathway is as follows: we
first base on the geometric view of Gaussian Processes (GP) to define
the concept of influence level that allows us to analytically represent
predictive means and variances of GP posteriors and base on that view
to enable replacing kernel similarity by a more genetic similarity
score. Furthermore, we also propose two strategies to find a batch
of query points that can efficiently handle high dimensional data.

%% file: GP_BayesOP.tex
In this section, we present GP-based Bayesian Optimization. The objective
is to minimize a \emph{black-box} function: $\text{max}_{x\in\mathcal{X}}\,f\left(x\right)$
where the feasible set $\mathcal{X}\subset\mathbb{R}^{d}$. At first,
we have not any collected data, we hence assume that $f$ is a random
function drawn from a Gaussian Process $\mathcal{GP}\left(\bzero,K\left(.,.\right)\right)$
(i.e., $f\sim\mathcal{GP}\left(\bzero,K\left(.,.\right)\right)$),
where $\bzero:\,\mathcal{X}\goto\mathbb{R}$ is the zero function
(i.e., $\bzero\left(x\right)=0,\,\forall x\in\mathcal{X}$), $K:\,\mathcal{X}\times\mathcal{X}\goto\mathbb{R}$
is a p.s.d kernel. Later at time $t$, assuming that we have collected
the argument-and-value set $\mathcal{D}_{t}=\left\{ \left(x_{1},y_{1}\right),\ldots,\left(x_{t-1},y_{t-1}\right)\right\} $
wherein each $y_{i}=f\left(x_{i}\right)+\varepsilon_{i}$ with $\varepsilon_{i}\sim\mathcal{N}\left(0,\sigma^{2}\right)$,
we are in need of specifying the next point $x_{t}$ to query. 

Given the set $\mathcal{D}_{t}$, the posterior $f^{\left(t\right)}=f\mid\mathcal{D}_{t}$
is the $\mathcal{GP}\left(\mu^{\left(t\right)}\left(x\right),K^{\left(t\right)}\left(.,.\right)\right)$
where $\mu^{\left(t\right)}:\mathcal{X}\backslash\mathcal{D}_{t}\goto\mathbb{R}$
and $K^{\left(t\right)}:\,\left(\mathcal{X}\backslash\mathcal{D}_{t}\right)\times\left(\mathcal{X}\backslash\mathcal{D}_{t}\right)\goto\mathbb{R}$
whose formulations are
\begin{gather}
\mu^{\left(t\right)}\left(x\right)=K_{x}^{\left(t\right)}\left[K_{\sigma}^{\left(tt\right)}\right]^{-1}\by^{\left(t\right)}\nonumber \\
K^{\left(t\right)}\left(x,x^{'}\right)=K\left(x,x^{'}\right)-K_{x}^{\left(t\right)}\left[K_{\sigma}^{\left(tt\right)}\right]^{-1}\transp{\left(K_{x^{'}}^{\left(t\right)}\right)}\nonumber \\
=\left|K\left(x,x^{'}\right)-K_{x}^{\left(t\right)}\left[K_{\sigma}^{\left(tt\right)}\right]^{-1}\transp{\left(K_{x^{'}}^{\left(t\right)}\right)}\right|\label{eq:GP_variance}
\end{gather}
where $K_{x}^{\left(t\right)}=\left[K\left(x,x_{i}\right)\right]_{i=1}^{t-1}$,
$K^{\left(tt\right)}=\left[K\left(x_{i},x_{j}\right)\right]_{i,j=1}^{t-1}$,
$K_{\sigma}^{\left(tt\right)}=K^{\left(tt\right)}+\sigma^{2}\mathbb{I}$,
and $\by^{\left(t\right)}=\transp{\left[y_{i}\right]_{i=1,...,t-1}}$.

Therefore, given any $x\in\mathcal{X}\backslash\mathcal{D}_{t}$,
$f^{\left(t\right)}\left(x\right)$ is a Gaussian random variable
with the mean and the standard deviation as $\mu^{\left(t\right)}\left(x\right)$
and $\sigma^{\left(t\right)}\left(x\right)=V^{\left(t\right)}\left(x\right)^{1/2}=K^{\left(t\right)}\left(x,x\right)^{1/2}$,
respectively. The principle to choose the next query point $x_{t}$
is to balance the\emph{ exploitation} against the \emph{exploration}.
The exploitation level of the point $x$ is expressed via the value
of $\mu^{\left(t\right)}\left(x\right)$ and its exploration level
is represented through the value of $\sigma^{\left(t\right)}\left(x\right)$.
Therefore, the next query point $x^{(t)}$ is evaluated as
\begin{equation}
x^{\left(t\right)}=\text{argmax}_{x}\,\left(\mu^{\left(t\right)}\left(x\right)+\kappa\sigma^{\left(t\right)}\left(x\right)\right)\label{eq:UCB}
\end{equation}

The above expression implies that we wish to minimize the mean $\mu^{\left(t\right)}\left(x\right)$
for the exploitation and simultaneously maximize the variance $\sigma^{(t)}\left(x\right)$
for the exploration. The exploitation and exploration is trade-off
since if we favor the exploitation, the query point tends to stay
close to the previous query points, hence having a small variance
(i.e., the standard deviation); in contrast, if we favor the exploration,
the query point tends to stay far away the previous query points for
a high variance (i.e., the standard deviation), hence having a low
mean value. Here we note that $\kappa>0$ is used to trade-off the
exploitation against the exploration.

To observe the geometric nature of GP-based Bayesian Optimization
(BO), we now investigate the geometric view of GP-based BO. Since
$\tilde{K}\left(x,x^{'}\right)=K\left(x,x^{'}\right)+\sigma\mathbb{I}\left(x,x^{'}\right)$
is a p.s.d kernel, there exists a feature map $\tilde{\Phi}:\mathcal{X}\goto\mathcal{H}$
(i.e., $\mathcal{H}$ is a Reproducing Kernel Hilbert Space) such
that $\tilde{K}\left(x,x^{'}\right)=\transp{\tilde{\Phi}\left(x\right)}\tilde{\Phi}\left(x^{'}\right)$.
We now denote $\mathcal{L}^{\left(t\right)}=\text{span}\left(\left\{ \tilde{\Phi}\left(x_{1}\right),\ldots,\tilde{\Phi}\left(x_{t-1}\right)\right\} \right)$
by the linear span of $\tilde{\Phi}\left(x_{1}\right),\ldots,\tilde{\Phi}\left(x_{t-1}\right)$
and further define the projection of a given vector $\tilde{\Phi}\left(x\right)$
onto $\mathcal{L}^{\left(t\right)}$ and the rejection of $\tilde{\Phi}\left(x\right)$
from $\mathcal{L}^{\left(t\right)}$ as
\begin{align*}
\mathcal{P}^{\left(t\right)}\left(x\right) & =\sum_{i=1}^{t-1}p_{i}\left(x\right)\tilde{\Phi}\left(x_{i}\right)\\
\mathcal{R}^{\left(t\right)}\left(x\right) & =\tilde{\Phi}\left(x\right)-\mathcal{P}^{\left(t\right)}\left(x\right)
\end{align*}
\begin{thm}
\label{thm:geo_view}(Geometric view) We define the coefficient vector
of the projection $\mathcal{P}^{\left(t\right)}\left(x\right)$ as
$p\left(x\right)=\transp{\left[p_{i}\left(x\right)\right]}_{i=1,...,t-1}$.
We then have $p\left(x\right)=K_{x}^{\left(t\right)}\left[K_{\sigma}^{\left(tt\right)}\right]^{-1}$.
In addition, the variance $V^{\left(t\right)}\left(x\right)$ is exactly
$\norm{\mathcal{R}^{\left(t\right)}\left(x\right)}-\sigma$, where
$\norm{\mathcal{R}^{\left(t\right)}\left(x\right)}$ is the Euclidean
distance from $\tilde{\Phi}\left(x\right)$ to the linear span $\mathcal{L}^{\left(t\right)}$
and the mean $\mu^{\left(t\right)}\left(x\right)$ is exactly $\left\langle p\left(x\right),\boldsymbol{y}^{\left(t\right)}\right\rangle =\sum_{i=1}^{t-1}p_{i}\left(x\right)y_{i}$.
\end{thm}

We now restate the criterion to find the next query point as shown
in Eq. (\ref{eq:UCB}) as
\[
x_{t+1}=\text{argmax}_{x}\,\left(\sum_{i=1}^{t-1}p_{i}\left(x\right)y_{i}+\kappa\sqrt{\norm{\mathcal{R}^{\left(t\right)}\left(x\right)}-\sigma}\right)
\]
This view supports us to think out of the GP-based Bayesian Optimization.
In particular, we propose a novel Similarity-based Bayesian Optimization
framework that still preserves the insightful spirit of the GP-based
Bayesian Optimization.

%% file: Sim_BayesOP.tex
\subsection{Thinking \emph{Out} of the Gaussian Process}

The kernel function $K\left(.,.\right)$ can be thought as a similarity
score which measures the similarity level between any two points.
Leveraging this remark with the geometric view of GP-based Bayesian
Optimization inspires us to think out of the Gaussian Process. In
particular, we propose a \emph{Analogical-based Bayesian Optimization}
(ABO) for which the kernel similarity can be replaced by a more generic
class of similarity scores. To motivate this idea, we observe that
the predictive mean can be computed as follows
\begin{equation}
\mu^{\left(t\right)}\left(x\right)=\sum_{i=1}^{t-1}p_{i}\left(x\right)y_{i}\label{eq:pred_mean}
\end{equation}
where each $p_{i}\left(x\right)$ stands for the coefficient of $\tilde{\Phi}\left(x_{i}\right)$
in the projection of $\tilde{\Phi}\left(x\right)$ onto $\mathcal{L}^{\left(t\right)}=\text{span}\left(\left\{ \tilde{\Phi}\left(x_{1}\right),...,\tilde{\Phi}\left(x_{t-1}\right)\right\} \right)$. 

The formula in Eq. (\ref{eq:pred_mean}) and the expressive meaning
of $p_{i}\left(x\right)$ enables us to assign $p_{i}\left(x\right)$
as the \emph{influence level} of $x_{i}$ to $x$ given $\mathcal{D}_{t}$
for which we denote as $I\left(x,x_{i}\mid\mathcal{D}_{t}\right)$.
If this influence level is high (i.e., $\tilde{\Phi}\left(x_{i}\right)$
plays an important role in the formula of $\mathcal{P}^{\left(t\right)}\left(x\right)$
or $\tilde{\Phi}\left(x\right)$), the collected value $y_{i}$ associating
with $\tilde{\Phi}\left(x_{i}\right)$ highly affects to the predictive
mean $\mu^{\left(t\right)}\left(x\right)$. With the notion of the
influence level in hand, we can rewrite the formula for the predictive
mean as
\begin{equation}
\mu^{\left(t\right)}\left(x\right)=\sum_{i=1}^{t-1}I\left(x,x_{i}\mid\mathcal{D}_{t}\right)y_{i}\label{eq:mean_influence}
\end{equation}

We now turn to express the variance $V^{\left(t\right)}\left(x\right)$
(or the standard deviation $\sigma^{\left(t\right)}\left(x\right)=V^{\left(t\right)}\left(x\right)^{1/2}$)
using the notion of the influence level. Using the formula in Eq.
(\ref{eq:GP_variance}), we can rewrite the variance as
\begin{equation}
V^{\left(t\right)}\left(x\right)=\left|K\left(x,x\right)-\sum_{i=1}^{t-1}I\left(x,x_{i}\mid\mathcal{D}_{t}\right)K\left(x,x_{i}\right)\right|\label{eq:variance_influence}
\end{equation}

The formula of the variance $V^{\left(t\right)}\left(x\right)$ in
Eq. (\ref{eq:variance_influence}) discloses that if $x$ locates
in the region highly affected by $x_{i}$ (s) and being close to $x_{i}$
(s), its variance would be low. In contrast, if $x$ tends to move
further away $x_{i}$ (s), its variance tends to decrease. Therefore,
in GP-based Bayesian Optimization, Gaussian Process allows us to place
the uncertainty over the ground-truth function $f$ and also quantitatively
characterize the uncertainty of this function evaluated at a point
(i.e., $f\left(x\right)$) which is influenced by other queried points
as in Eqs. (\ref{eq:mean_influence}, \ref{eq:variance_influence}). 

\subsection{Bayesian Optimization with a Generic Similarity Score}

With the support of the above views and reasons, we propose to replace
the kernel function $K\left(x,x^{'}\right)$ by a more generic similarity
score $S\left(x,x^{'}\right)$ wherein $S:\,\mathcal{X}\times\mathcal{X}\goto\mathbb{R}$
is non-negative and symmetric. The formulas for the predictive mean
and variance as shown in Eqs. (\ref{eq:mean_influence}, \ref{eq:variance_influence})
are rewritten as
\begin{align*}
\mu^{\left(t\right)}\left(x\right) & =\sum_{i=1}^{t-1}I\left(x,x_{i}\mid\mathcal{D}_{t}\right)y_{i}\\
V^{\left(t\right)}\left(x\right) & =\left|S\left(x,x\right)-\sum_{i=1}^{t-1}I\left(x,x_{i}\mid\mathcal{D}_{t}\right)S\left(x,x_{i}\right)\right|
\end{align*}
where $I\left(x\right)=\transp{\left[I\left(x,x_{i}\mid\mathcal{D}_{t}\right)\right]}_{i=1,..,t-1}$
can be computed as $S_{x}^{\left(t\right)}\left[S^{\left(tt\right)}+\sigma^{2}\mathbb{I}\right]^{-1}$
(if available) with $S_{x}^{\left(t\right)}=\left[S\left(x,x_{i}\right)\right]_{i=1}^{t-1}$
and $S^{\left(tt\right)}=\left[S\left(x_{i},x_{j}\right)\right]_{i,j=1}^{t-1}$.

However, for a generic similarity score $S\left(.,.\right)$, the
matrix $S^{\left(tt\right)}+\sigma^{2}\mathbb{I}$ might be a singular
matrix, hence making the computation infeasible. To address this issue,
we note that $I\left(x\right)=S_{x}^{\left(t\right)}\left[S^{\left(tt\right)}+\sigma^{2}\mathbb{I}\right]^{-1}$
or equivalently $S_{x}^{\left(t\right)}=I\left(x\right)\times\left(S^{\left(tt\right)}+\sigma^{2}\mathbb{I}\right)$
and therefore propose to find $I\left(x\right)$ as
\begin{equation}
I\left(x\right)=\text{argmin}_{I}\,\norm{S_{x}^{\left(t\right)}-I\left(S^{\left(tt\right)}+\sigma^{2}\mathbb{I}\right)}^{2}\label{eq:I_OP}
\end{equation}

To find optimal solution of the optimization problem in Eq. (\ref{eq:I_OP}),
we denote $r_{t}=\text{rank}\left(S^{\left(tt\right)}+\sigma^{2}\mathbb{I}\right)$
and let $B^{\left(tt\right)}$ be the base matrix of the row space
of the matrix $S^{\left(tt\right)}+\sigma^{2}\mathbb{I}$. It is apparent
that the size of $B^{\left(tt\right)}$ is $r_{t}\times\left(t-1\right)$
which depends on the similarity score $S\left(.,.\right)$. The following
theorem states that instead of solving the optimization problem in
Eq. (\ref{eq:I_OP}) we can solve a similar optimization with a smaller
size.
\begin{thm}
(Equivalent problem) Let us denote $J\left(x\right)=\text{argmin}_{J}\,\norm{S_{x}^{\left(t\right)}-JB^{\left(tt\right)}}^{2}\in\mathbb{R}^{r_{t}}$.
The following statements hold

i) The matrix $B^{\left(tt\right)}\transp{\left(B^{\left(tt\right)}\right)}$
is invertible. 

ii) $J\left(x\right)=S_{x}^{\left(t\right)}\transp{\left(B^{\left(tt\right)}\right)}\left[B^{\left(tt\right)}\transp{\left(B^{\left(tt\right)}\right)}\right]^{-1}$.
$I\left(x\right)$ can be formed by augmenting $J\left(x\right)$
with the zero entries.
\end{thm}

It is apparent that if the matrix $S^{\left(tt\right)}+\sigma^{2}\mathbb{I}$
is invertible (i.e., $\text{rank}\left(S^{\left(tt\right)}+\sigma^{2}\mathbb{I}\right)=t-1$)
and symmetric, we can gain the formulation being similar to GP-based
OP as shown in the following corollary.
\begin{cor}
Assuming that the matrix $S^{\left(tt\right)}+\sigma^{2}\mathbb{I}$
is invertible (i.e., $\text{rank}\left(S^{\left(tt\right)}+\sigma^{2}\mathbb{I}\right)=t-1$)
and symmetric, we then have $B^{\left(tt\right)}=S^{\left(tt\right)}+\sigma^{2}\mathbb{I}$,
and $I\left(x\right)=J\left(x\right)=S_{x}^{\left(t\right)}\left(S^{\left(tt\right)}+\sigma^{2}\mathbb{I}\right)^{-1}$.
\end{cor}

It is worth noting that $J\left(x\right)$ does not match with $\by^{\left(t\right)}$
and $S_{x}^{\left(t\right)}$ in general. To make the computation
tractable, we fill the missing values in $J\left(x\right)$ by $0$.
As a sequence, the calculations of $S_{x}^{\left(t\right)}D^{\left(tt\right)}\by^{\left(t\right)}$
and $S_{x}^{\left(t\right)}D^{\left(tt\right)}\transp{\left(S_{x}^{\left(t\right)}\right)}$
can be realized by eliminating the irrelevant entries in $\by^{\left(t\right)}$
and $S_{x}^{\left(t\right)}$. 

In the sequel, we demonstrate that the influence vector evaluated
as in Eq. (\ref{eq:I_OP}) has the same geometric interpretation as
that of GP-based Bayesian Optimization in Theorem \ref{thm:geo_view}.
The only difference is that the \emph{empirical} feature map is used
instead of the feature map $\tilde{\Phi}\left(.\right)$. Given the
collected training set $\mathcal{D}_{t-1}$, the empirical feature
map is defined as
\[
\Phi_{e}\left(x\right)=\left[S\left(x,x_{i}\right)+\sigma^{2}\mathbb{I}\left(x,x_{i}\right)\right]_{i=1}^{t-1}
\]

The following theorem shows that the influence vector evaluated as
in Eq. (\ref{eq:I_OP}) is exactly the coefficients of the vectors
$\Phi_{e}\left(x_{i}\right)$(s) in the projection of $\Phi_{e}\left(x\right)$
onto the linear span of $\left\{ \Phi_{e}\left(x_{1}\right),\ldots,\Phi_{e}\left(x_{t-1}\right)\right\} $.
\begin{thm}
\label{thm:empirical_projection}(Geometric view with empirical feature
map) Let us denote the projection of $\Phi_{e}\left(x\right)$ onto
the linear span of $\left\{ \Phi_{e}\left(x_{1}\right),\ldots,\Phi_{e}\left(x_{t-1}\right)\right\} $
by $\mathcal{P}_{e}\left(x\right)$. Let $I_{i}\left(x\right)$ be
the $i$-th component of the influence vector evaluated as in Eq.
(\ref{eq:I_OP}). We then have
\[
\mathcal{P}_{e}\left(x\right)=\sum_{i=1}^{t-1}I_{i}\left(x\right)\Phi_{e}\left(x_{i}\right)
\]
\end{thm}

Theorem \ref{thm:empirical_projection} indicates that the influence
vector evaluated as in Eq. (\ref{eq:I_OP}) preserves the key spirit
of the influence concept in GP-based Bayesian Optimization.

\subsection{Acquisition Function and Strategy to Query }

In this section, we present two kinds of acquisition function and
the strategy to find a batch of query points. The maximization of
the proposed acquisition functions is based on the fixed-point technique
wherein each point in the current queried set has its own trajectory
to gradually converge to an equilibrium point, which is also a local
maxima of the current acquisition function. Two proposed acquisition
functions are formulated as 
\begin{align*}
u_{1}^{\left(t\right)}\left(x\right) & =S_{x}^{\left(t\right)}D^{\left(tt\right)}\by^{\left(t\right)}\\
u_{2}^{\left(t\right)}\left(x\right) & =S_{x}^{\left(t\right)}D^{\left(tt\right)}\by^{\left(t\right)}+\kappa\left|S\left(x,x\right)-S_{x}^{\left(t\right)}D^{\left(tt\right)}\transp{\left(S_{x}^{\left(t\right)}\right)}\right|^{1/2}
\end{align*}

To maximize the above acquisition functions, we use the fixed point
technique. In particular, we need to find an equilibrium point such
that $\nabla u\left(x^{*}\right)=0$ or $\nabla u\left(x^{*}\right)+x^{*}=x^{*}$
where $u\left(x\right)$ can be $u_{1}^{\left(t\right)}\left(x\right)$
or $u_{2}^{\left(t\right)}\left(x\right)$. To address it, we define
$g\left(x\right)=\nabla u\left(x\right)+x$ and start with an initial
point $x^{\left(0\right)}$, and then find the next point as $x^{\left(l+1\right)}=g\left(x^{\left(l\right)}\right)$.
This sequence will converge to an equilibrium point $\text{equi}\left(x^{\left(0\right)}\right)$. 

We now respectively debut with $x_{1},x_{2},...,x_{t-1}$ as initial
points (i.e., $x_{i}=x^{\left(0\right)},\,i=1,...,t-1$ respectively).
The $x_{i}$ (s) converge to the equilibrium points $\text{equi}\left(x_{i}\right)$
(s) and some of them might be coincided. We now define the set of
equilibrium points by $EQ^{\left(t\right)}$ (i.e., $\left|EQ^{\left(t\right)}\right|\leq t-1$).
Given a batch size $n_{b}$, with the first strategy we choose the
top $n_{b}$ equilibrium points with highest predictive variance (i.e.,
$\left|S\left(x,x\right)-S_{x}^{\left(t\right)}D^{\left(tt\right)}\transp{\left(S_{x}^{\left(t\right)}\right)}\right|$)
and with the second strategy we choose the top $n_{b}$ equilibrium
points with highest objective value (i.e., $u_{2}^{\left(t\right)}\left(x\right)$).
In addition, in the first strategy we propose the two-stage strategy
wherein the first stage bases on exploitation and the second stage
bases on exploration. The gradient (or subgradient) of $u_{1}^{\left(t\right)}\left(x\right)$
and $u_{2}^{\left(t\right)}\left(x\right)$ (or $g\left(x\right)=\nabla u\left(x\right)+x$)
can be conveniently computed as follows
\begin{gather*}
\nabla u_{1}^{\left(t\right)}\left(x\right)=\nabla S_{x}^{\left(t\right)}D^{\left(tt\right)}\by^{\left(t\right)}\\
\nabla u_{2}^{\left(t\right)}\left(x\right)=\nabla u_{1}^{\left(t\right)}\left(x\right)+\frac{\kappa\nabla S\left(x,x\right)}{2V^{\left(t\right)}\left(x\right)\text{sign}\left(V^{\left(t\right)}\left(x\right)\right)}-\frac{\kappa\left(\nabla S_{x}^{\left(t\right)}D^{\left(tt\right)}\transp{\left(S_{x}^{\left(t\right)}\right)}+S_{x}^{\left(t\right)}D^{\left(tt\right)}\transp{\left(\nabla S_{x}^{\left(t\right)}\right)}\right)}{2V^{\left(t\right)}\left(x\right)\text{sign}\left(V^{\left(t\right)}\left(x\right)\right)}
\end{gather*}

%% file: Application.tex
In this section, we present a typical example optimization problem
wherein the existing approaches are infeasible to accurately solve
it whilst our proposed $\model$ can efficiently figure out its solution.
Assuming that we are dealing with the following optimization problem:
\[
\text{max}_{x}\,g\left(x\right)\triangleq\mathbb{E}_{p\left(\omega\mid x\right)}\left[f\left(\omega,x\right)\right]
\]

In the above optimization problem, the formula of the function $f\left(\omega,x\right)$
is clear, but the evaluation of the expectation is intractable. Therefore,
we consider the function $g\left(x\right)$ as a \emph{black box}
function. Given $x$, we can use Monte Carlo (MC) estimation to evaluate
$g\left(x\right)$ using $\omega_{i}$ (s) drawn from $p\left(\omega\mid x\right)$.
Certainly, we are free to employ the traditional GP-based BO in this
case. However, the \emph{Gaussian} kernel function of this approach
is based on the Euclidean (or Mahalanobis) distance, hence entailing
unsatisfied solution. It is more appealing if we recruit the symmetric
KL divergence to measure similarity score as between $x,x^{'}$ as
follows{\small{}
\begin{gather*}
S\left(x,x^{'}\right)=\text{const}-D_{SYM}\left(x,x^{'}\right)=\text{const}-\\
\frac{D_{KL}\left(p\left(.\mid x\right)\parallel p\left(.\mid x^{'}\right)\right)}{2}-\frac{D_{KL}\left(p\left(.\mid x^{'}\right)\parallel p\left(.\mid x\right)\right)}{2}
\end{gather*}
}{\small \par}

The derivative of $S\left(x,x^{'}\right)$ w.r.t $x$ is as follows
\begin{gather}
\nabla_{x}S\left(x,x^{'}\right)=\int\nabla_{x}\log\,p\left(.\mid x\right)p\left(.\mid x^{'}\right)d\omega\nonumber \\
-\int\nabla_{x}\log p\left(.\mid x\right)\log\frac{\exp\left(1\right)p\left(.\mid x\right)}{p\left(.\mid x^{'}\right)}p\left(.\mid x\right)d\omega\label{eq:derivative}
\end{gather}

It is obvious that in case that the evaluation of the derivative in
Eq. (\ref{eq:derivative}) is intractable, we can estimate it using
MC estimation. Therefore, in general the execution of $\model$ for
the above Bayesian optimization problem is always feasible. To demonstrate
the idea and simplify the problem, we assume that $x=\left(\mu,\Sigma\right)$
and $p\left(\omega\mid x\right)=\mathcal{N}\left(\omega\mid\mu,\Sigma\right)$
where $\Sigma=\diag\left(\left[\sigma_{i}\right]_{i=1}^{d}\right)$.
We then have{\small{}
\begin{gather*}
S\left(x,x^{'}\right)=\text{const}-\frac{1}{4}\left(\trace\left(\Sigma^{-1}\Sigma'\right)+\trace\left(\Sigma\Sigma'^{-1}\right)\right)\\
-\frac{1}{4}\transp{\left(\mu-\mu^{'}\right)}\left(\Sigma^{-1}+\Sigma'^{-1}\right)\left(\mu-\mu^{'}\right)\\
=\text{const}-\frac{1}{4}\sum_{i=1}^{d}\left(\frac{\sigma_{i}}{\sigma_{i}'}+\frac{\sigma_{i}'}{\sigma_{i}}\right)-\frac{1}{4}\sum_{i=1}^{d}\left(\mu_{i}-\mu_{i}^{'}\right)^{2}\left(\frac{1}{\sigma_{i}}+\frac{1}{\sigma_{i}'}\right)
\end{gather*}
where $x'=\left(\mu',\Sigma'\right)$ and $\Sigma'=\diag\left(\left[\sigma_{i}'\right]_{i=1}^{d}\right)$. }{\small \par}

The derivative is now tractable as follows
\begin{align*}
\nabla S_{\mu_{i}} & =-\frac{1}{2}\left(\mu_{i}-\mu_{i}'\right)\left(\frac{1}{\sigma_{i}}+\frac{1}{\sigma_{i}'}\right)\\
\nabla S_{\sigma_{i}} & =-\frac{1}{4}\left(\frac{1}{\sigma_{i}'}-\frac{\sigma_{i}'}{\sigma_{i}^{2}}\right)-\frac{1}{4}\frac{\left(\mu_{i}-\mu_{i}'\right)^{2}}{\sigma_{i}^{2}}
\end{align*}

%% file: icml17_supp.tex
\section{All Proofs}

For comprehensibility, we first revise some definitions and notations
used in the paper.
\begin{align*}
\tilde{K}\left(x,x^{'}\right) & =K\left(x,x^{'}\right)+\sigma\mathbb{I}\left(x,x^{'}\right)\,\text{and}\,\tilde{K}\left(x,x^{'}\right)=\transp{\tilde{\Phi}\left(x\right)}\tilde{\Phi}\left(x^{'}\right)\\
\mathcal{L}^{\left(t\right)} & =\text{span}\left(\left\{ \tilde{\Phi}\left(x_{1}\right),\ldots,\tilde{\Phi}\left(x_{t-1}\right)\right\} \right)\\
\mathcal{P}^{\left(t\right)}\left(x\right) & =\sum_{i=1}^{t-1}p_{i}\left(x\right)\tilde{\Phi}\left(x_{i}\right)\,\text{and}\,\mathcal{R}^{\left(t\right)}\left(x\right)=\tilde{\Phi}\left(x\right)-\mathcal{P}^{\left(t\right)}\left(x\right)
\end{align*}
\begin{thm}
We define the coefficient vector of the projection $\mathcal{P}^{\left(t\right)}\left(x\right)$
as $p\left(x\right)=\transp{\left[p_{i}\left(x\right)\right]}_{i=1,...,t-1}$.
We then have $p\left(x\right)=\left[K_{\sigma}^{\left(tt\right)}\right]^{-1}$$\transp{\left(K_{x}^{\left(t\right)}\right)}$
where $K_{\sigma}^{\left(tt\right)}=K^{\left(tt\right)}+\sigma^{2}\mathbb{I}$
with $K^{\left(tt\right)}=\left[K\left(x_{i},x_{j}\right)\right]_{i,j=1}^{t-1}$.
In addition, the variance $V^{\left(t\right)}\left(x\right)$ is exactly
$\norm{\mathcal{R}^{\left(t\right)}\left(x\right)}-\sigma$, where
$\norm{\mathcal{R}^{\left(t\right)}\left(x\right)}$ is the distance
from $\tilde{\Phi}\left(x\right)$ to the linear span $\mathcal{L}^{\left(t\right)}$
and the mean $\mu^{\left(t\right)}\left(x\right)$ is exactly $\left\langle p\left(x\right),\boldsymbol{y}^{\left(t\right)}\right\rangle =\sum_{i=1}^{t-1}p_{i}\left(x\right)y_{i}$.
\end{thm}

\begin{proof}
It is apparent that
\[
p\left(x\right)=\text{argmin}_{d}\,J\left(d\right)\triangleq\norm{\tilde{\Phi}\left(x\right)-\sum_{i=1}^{t-1}d_{i}\tilde{\Phi}\left(x_{i}\right)}^{2}
\]

We then have
\begin{align*}
J\left(d\right) & =\tilde{K}\left(x,x\right)-\sum_{i=1}^{t-1}\tilde{K}\left(x,x_{i}\right)d_{i}+\sum_{i=1}^{t-1}\sum_{j=1}^{t-1}d_{i}d_{j}\tilde{K}\left(x_{i},x_{j}\right)\\
 & =\tilde{K}\left(x,x\right)-\sum_{i=1}^{t-1}K\left(x,x_{i}\right)d_{i}+\sum_{i=1}^{t-1}\sum_{j=1}^{t-1}d_{i}d_{j}\tilde{K}\left(x_{i},x_{j}\right)\,\,(\text{since}\,x\neq x_{i},\,\forall i)\\
 & =\tilde{K}\left(x,x\right)-K_{x}^{\left(t\right)}\transp d+\transp d\left[K^{\left(tt\right)}+\sigma^{2}\mathbb{I}\right]d\\
 & =\tilde{K}\left(x,x\right)-K_{x}^{\left(t\right)}\transp d+\transp dK_{\sigma}^{\left(tt\right)}d
\end{align*}
\[
\nabla J\left(d\right)=\transp{\left(K_{x}^{\left(t\right)}\right)}-K_{\sigma}^{\left(tt\right)}d
\]

Setting the derivative to $\bzero$, we gain
\[
p\left(x\right)=d^{*}=\left[K_{\sigma}^{\left(tt\right)}\right]^{-1}\transp{\left(K_{x}^{\left(t\right)}\right)}
\]
\end{proof}
We now remind the formula to compute the influence vector $I\left(x\right)$

\begin{equation}
I\left(x\right)=\text{argmin}_{I}\,\norm{S_{x}^{\left(t\right)}-I\left(S^{\left(tt\right)}+\sigma^{2}\mathbb{I}\right)}^{2}\label{eq:I_OP}
\end{equation}
\begin{thm}
Let us denote $J\left(x\right)=\text{argmin}_{J}\,\norm{S_{x}^{\left(t\right)}-JB^{\left(tt\right)}}^{2}\in\mathbb{R}^{r_{t}}$.
The following statements hold

i) The matrix $B^{\left(tt\right)}\transp{\left(B^{\left(tt\right)}\right)}$
is invertible. 

ii) $J\left(x\right)=S_{x}^{\left(t\right)}\transp{\left(B^{\left(tt\right)}\right)}\left[B^{\left(tt\right)}\transp{\left(B^{\left(tt\right)}\right)}\right]^{-1}$.
$I\left(x\right)$ can be formed by augmenting $J\left(x\right)$
with the zero entries.
\end{thm}

\begin{proof}
We sketch out the proof as follows.

i) $\text{rank}\left(B^{\left(tt\right)}\transp{\left(B^{\left(tt\right)}\right)}\right)=\text{rank}\left(B^{\left(tt\right)}\right)=r_{t}$.
In addition, the size of the matrix $B^{\left(tt\right)}\transp{\left(B^{\left(tt\right)}\right)}$
is $r_{t}\times r_{t}$. It follows that this matrix is invertible.

ii) Setting the derivative of the objective function w.r.t $J$ to
$\bzero$, we gain
\begin{align*}
\bzero & =2\left(J\left(x\right)B^{\left(tt\right)}-S_{x}^{\left(t\right)}\right)\transp{\left(B^{\left(tt\right)}\right)}\\
J\left(x\right) & =S_{x}^{\left(t\right)}\transp{\left(B^{\left(tt\right)}\right)}\left[B^{\left(tt\right)}\transp{\left(B^{\left(tt\right)}\right)}\right]^{-1}
\end{align*}

According to the definition of $B^{\left(t\right)}$, we gain
\begin{align*}
\left\{ JB^{\left(tt\right)}:\,J\in\mathbb{R}^{r_{t}}\right\}  & =\left\{ I\left(S^{\left(tt\right)}+\sigma^{2}\mathbb{I}\right):\,I\in\mathbb{R}^{t-1}\right\} =C^{\left(t\right)}
\end{align*}

Therefore, we arrive at
\begin{align*}
\norm{S_{x}^{\left(t\right)}-I\left(x\right)\left(S^{\left(tt\right)}+\sigma^{2}\mathbb{I}\right)} & =\norm{S_{x}^{\left(t\right)}-J\left(x\right)B^{\left(tt\right)}}=\text{max}_{v\in C^{\left(t\right)}}\norm{S_{x}^{\left(t\right)}-v}
\end{align*}

It concludes this proof since $B^{\left(tt\right)}$ is a submatrix
of $S^{\left(tt\right)}+\sigma^{2}\mathbb{I}$.
\end{proof}
\begin{cor}
Assuming that the matrix $S^{\left(tt\right)}+\sigma^{2}\mathbb{I}$
is invertible (i.e., $\text{rank}\left(S^{\left(tt\right)}+\sigma^{2}\mathbb{I}\right)=t-1$)
and symmetric, we then have $B^{\left(tt\right)}=S^{\left(tt\right)}+\sigma^{2}\mathbb{I}$,
and $I\left(x\right)=J\left(x\right)=S_{x}^{\left(t\right)}\left(S^{\left(tt\right)}+\sigma^{2}\mathbb{I}\right)^{-1}$.
\end{cor}

\begin{proof}
We derive as
\begin{align*}
I\left(x\right) & =J\left(x\right)=S_{x}^{\left(t\right)}\transp{\left(B^{\left(tt\right)}\right)}\left[B^{\left(tt\right)}\transp{\left(B^{\left(tt\right)}\right)}\right]^{-1}=S_{x}^{\left(t\right)}B^{\left(tt\right)}\left[B^{\left(tt\right)}B^{\left(tt\right)}\right]^{-1}\\
 & =S_{x}^{\left(t\right)}B^{\left(tt\right)}\left(B^{\left(tt\right)}\right)^{-1}\left(B^{\left(tt\right)}\right)^{-1}=S_{x}^{\left(t\right)}\left(B^{\left(tt\right)}\right)^{-1}
\end{align*}
\end{proof}
\begin{thm}
\label{thm:empirical_projection}Let us denote the projection of $\Phi_{e}\left(x\right)$
onto the linear span of $\left\{ \Phi_{e}\left(x_{1}\right),\ldots,\Phi_{e}\left(x_{t-1}\right)\right\} $
by $\mathcal{P}_{e}\left(x\right)$. Let $I_{i}\left(x\right)$ be
the $i$-th component of the influence vector evaluated as in Eq.
(\ref{eq:I_OP}). We then have
\[
\mathcal{P}_{e}\left(x\right)=\sum_{i=1}^{t-1}I_{i}\left(x\right)\Phi_{e}\left(x_{i}\right)
\]
\end{thm}

\begin{proof}
We have
\[
\mathcal{P}_{e}\left(x\right)=\sum_{i=1}^{t-1}I_{i}\left(x\right)\Phi_{e}\left(x_{i}\right)
\]

It is apparent that
\[
I\left(x\right)=\text{argmin}_{I}\,\norm{\Phi_{e}\left(x\right)-\sum_{i=1}^{t-1}I_{i}\Phi_{e}\left(x_{i}\right)}^{2}
\]

We note that
\begin{eqnarray*}
\Phi_{e}\left(x\right) & = & \left[S\left(x,x_{i}\right)+\sigma^{2}\mathbb{I}\left(x,x_{i}\right)\right]_{i=1}^{t-1}=\left[S\left(x,x_{i}\right)\right]_{i=1}^{t-1}=S_{x}^{\left(t\right)}\,\,\,\,\,(\text{since}\,x\neq x_{i},\,\forall i)\\
\sum_{i=1}^{t-1}I_{i}\Phi_{e}\left(x_{i}\right) & = & I\transp{\left[\Phi_{e}\left(x_{i}\right)\right]_{i=1,...,t-1}}=I\left[\Phi_{e}\left(x_{i}\right)\right]_{i=1}^{t-1}=I\left(S^{\left(tt\right)}+\sigma^{2}\mathbb{I}\right)\,\,\,\,\,(\text{since}\,\left[\Phi_{e}\left(x_{i}\right)\right]_{i=1}^{t-1}\,\text{is\,symmetric})
\end{eqnarray*}

Therefore, we gain the conclusion.
\end{proof}